\documentclass[conference,letterpaper]{IEEEtran}

\usepackage{color}
\usepackage{epsfig}
\usepackage{amsmath}
\usepackage{amsfonts}
\usepackage{amsthm}
\usepackage{amssymb}
\usepackage{accents}
\usepackage{graphics}
\usepackage{graphicx}
\usepackage{float}
\usepackage{url}

\newtheorem*{example}{Example 1}
\newtheorem{definition}{Definition}
\newtheorem{theorem}{Theorem}
\newtheorem{lemma}{Lemma}

\newtheorem{corollary}{Corollary}

\newcommand{\Z}{\mathbb{Z}}
\newcommand{\R}{\mathbb{R}}

\newcommand{\B}{\mathcal{B}}
\newcommand{\F}{\mathcal{F}}
\newcommand{\D}{\mathcal{D}}

\newcommand{\PPB}{\mathcal{P}(\mathcal{B})}
\newcommand{\PPBbar}{\overline{\mathcal{P}}(\mathcal{B})}

\newcommand{\SB}{\mathcal{S}(\mathcal{B})}
\newcommand{\CB}{\mathcal{C}(\mathcal{B})}

\newcommand{\CBONE}{\mathcal{C}^{1}_{\PPB}}
\newcommand{\CBZERO}{\mathcal{C}^{0}_{\PPB}}

\newcommand{\hx}{\hat{x}}
\newcommand{\hz}{\hat{z}}

\newcommand{\bias}{p}

\newcommand{\CN}{\tau_f}

\title{\LARGE \bf On the CVP for the root lattices \\
                          via folding with deep ReLU neural networks}

\author{Vincent Corlay$^{\dagger,*}$, Joseph J. Boutros$^{\ddagger}$, Philippe Ciblat$^{\dagger}$, and Lo\"ic Brunel$^*$\\
$^{\dagger}$ Telecom ParisTech, 46 Rue Barrault, 75013 Paris, France, v.corlay@fr.merce.mee.com\\
$^{\ddagger}$ Texas A\&M University, Doha, Qatar, $^*$Mitsubishi Electric R\&D Centre Europe, Rennes, France
}

\begin{document}

\maketitle
\thispagestyle{plain}
\pagestyle{plain}

\begin{abstract}
Point lattices and their decoding via neural networks are considered in this paper.
Lattice decoding in $\R^n$, known as the closest vector problem (CVP),
becomes a classification problem in the fundamental parallelotope
with a piecewise linear function defining the boundary.
Theoretical results are obtained by studying root lattices. 
We show how the number of pieces in the boundary function reduces dramatically with folding,
from exponential to linear. This translates into a two-layer ReLU network
requiring a number of neurons growing exponentially in $n$ to solve the CVP,
whereas this complexity becomes polynomial in $n$ for a deep ReLU network. 
\end{abstract}

\section{Introduction and motivations}
The objective  of this paper is  two-fold. Firstly, we introduce  a new
paradigm to  solve the CVP. This approach
enables to find efficient decoding algorithms for some dense lattices.
For  instance, such a neural network for the Gosset lattice
is a key component of a neural Leech decoder. Secondly, we also  aim at
contributing to the understanding of the efficiency of deep learning,
namely the expressive power of  deep neural networks.  As a result,
our goal is to present new decoding algorithms  and interesting
functions that can be efficiently computed by deep neural networks.

Deep Learning is  about two key aspects: (i) Finding  a function class
$\Phi=\{f\}$ that contains a function $f^*$ ``close enough'' to a target function.
(ii) Finding a learning algorithm $L$  for the class $\Phi$.
Of course the choices of (i) and (ii) can be either done jointly or
separately but in either case they impact each other.
Research  on the  expressive  power of  deep  neural networks  focuses
mostly on  (i)~\cite{Eldan2016}\cite{Raghu2016},
by studying  some specific functions contained  in the
function class of a network. Typically,  the aim is to show that there
exist functions that can be well approximated by a deep network with a
polynomial  number  of parameters  whereas  an  exponential number  of
parameters is  required for a shallow  network.
This  line of work leads to ``gap" theorems and ``capacity''
theorems for deep networks and it is similar to the  classical theory of
Boolean circuit complexity~\cite{Hastad1986}.

In  this  scope, several  papers  investigate  specifically deep  ReLU
networks \cite{Pascanu2013}\cite{Montufar2014}\cite{Raghu2016}\cite{Telgarsky2016}\cite{Safran2017}\cite{Arora2018}
(See \cite[Section~2.1]{Montufar2014} for a short introduction to ReLU neural networks).
Since a ReLU network computes a
composition  of   piecewise affine  functions,  all  functions   in  $\Phi$  are
continuous  piecewise  linear   (CPWL).
Hence,  the efficiency of $\Phi$ can be evaluated by checking whether
a CPWL function with a lot of affine pieces belongs to $\Phi$.
For example, there exists at least  one function  in
$\R^n$  with $\Omega  \left( \left(  w/n \right)^{L-1}  w^n \right)$
affine pieces that can be computed with a $w$-wide deep ReLU network
having $L$ hidden layers~\cite{Montufar2014}.
A two-layer network would need an exponential number of parameters for this same function.
In~\cite{Raghu2016}  they further
show that any random deep network achieves a similar exponential behavior.


Some results in the literature are established by considering elementary oscillatory functions (see e.g. Appendix~\ref{sec_montu})
or piecewise linear functions with regions of random shape.
It is not clear whether such types of functions may arise naturally in
computer science and engineering fields.



Our  work lies somewhere between \cite{Montufar2014}\cite{Telgarsky2016}
and \cite{Raghu2016}: our functions are neither elementary nor random.
We  discovered them  in  the  context of sphere packing and lattices
which are solution to many fundamental problems in number theory,
chemistry, communication theory, string theory, and cryptography \cite{Conway1999}.
Hence, in contrary to existing works, we do not search for a specific class  of functions to justify the success of
deep learning. We set technical problems, study the functions arising from lattices,
and show  that deep networks  are suited to  tackle them as  we obtain
similar  gap   theorems  between  shallow  and   deep  models.


\section{Main results}
A ReLU network with a finite number of neurons is not capable 
to infer infinite periodic functions (see Appendix~\ref{sec_montu}).
Hence, it cannot implement a simple modulo operation.
As a result, we allow a low complexity pre-processing of the point
to be decoded to obtain an equivalent position in the fundamental parallelotope $\PPB$
of the lattice.

\begin{enumerate}
\item Theorems~\ref{th_func_VR}\&\ref{th_func_VR_2} show that the decision boundary for the hyperplane logical decoder
(HLD) \cite{Corlay2018} is given by a continuous piecewise linear function,
for any lattice with a Voronoi-reduced basis.
Corollary~\ref{coro_SVR} guarantees the same result for a semi-Voronoi-reduced basis.
\item For the lattice $A_n$ with a basis defined by the Gram matrix~$\eqref{eq_basis_An}$
and a point in $\PPB$, Theorem~\ref{theo_nbReg_Lin} proves that the decision function has $\Omega(2^n)$ affine pieces.
\item Also for $A_n$, the number of pieces is reduced to $\mathcal{O}(n)$ after folding
as stated in Theorem~\ref{theo_An_linear}. 
\item Results of Section~\ref{sec_folding_An}~and~\ref{sec_folding_relu}, 
based on Theorem~\ref{theo_An_linear}, implies  that there exists a ReLU 
network of depth $\mathcal{O}\left(n^2\right)$ and width $\mathcal{O}\left(n^2\right)$
solving the CVP.
\item Theorem~\ref{theo_shallow} shows that a ReLU network with only one hidden layer needs  $\Omega\left(2^{n}\right)$ neurons to solve the CVP.
\end{enumerate}
Moreover, the theory presented in this paper is not limited to~$A_n$. 
It extends very well to other dense lattices.
Indeed, we already obtained similar results for all root lattices. 
They will however be presented in future communications due to lack of space. 
Finally, this paradigm seems not to be limited to lattices
as it may extend to binary block codes (see e.g. Figure~\ref{fig_parity} in Appendix~\ref{sec_montu}).

\section{Lattices and polytopes}
A lattice $\Lambda$ is a discrete additive subgroup of $\R^n$.
For a rank-$n$ lattice in $\R^n$, the rows of a $n\times n$ generator matrix $G$ constitute
a basis of $\Lambda$ and any lattice point $x$ is obtained via $x=zG$, where $z \in \Z^n$.
For a given basis $\mathcal{B}=\{ b_i \}_{i=1}^n$, $\PPB$ denotes
the fundamental parallelotope of $\Lambda$ and $\mathcal{V}(x)$ the Voronoi cell
of a lattice point $x$ \cite{Corlay2018}.
The minimum Euclidean distance of $\Lambda$ is $d_{min}(\Lambda)=2\rho$, where $\rho$ is the packing radius.

A vector $v \in \Lambda$ is called Voronoi vector if the half-space $\{y \in \mathbb{R}^{n} \ : \ y \cdot v \leq \frac{1}{2}v \cdot v \}$ 
has a non empty intersection with $\mathcal{V}(0)$. The vector is said relevant if the intersection is an $n-1$-dimensional face of $\mathcal{V}(0)$.
We denote by $\tau_{f}$ the number of relevant Voronoi vectors, referred to as the {\em Voronoi number} in the sequel.
For root lattices \cite{Conway1999}, the Voronoi number is equal to the kissing number $\tau$.
For random lattices, we typically have $\tau_{f}=2^{n+1}-2$.
The set $\CN(x)$, for $x \in \Lambda$, is the set of lattice points having a common Voronoi facet with $x$.
The next definition, introduced in \cite{Corlay2018}, is important for the rest of the paper.
\begin{definition}
\label{def_Voronoi-reduced}
Let $\B$ be the $\Z$-basis of a rank-$n$ lattice $\Lambda$ in~$\R^n$.
$\B$ is said Voronoi-reduced (VR) if, for any point $y \in \PPB$,
the closest lattice point $\hx$ to $y$ is one of the $2^n$ corners of $\PPB$,
i.e. $\hx=\hz G$ where $\hz \in \{0, 1\}^n$.
\end{definition}


Lattice decoding refers to finding the closest lattice point,
the closest in Euclidean distance sense.
This problem is also known as the closest vector problem.
The neural lattice decoder employs $\PPB$ as its main compact region \cite{Corlay2018},
thus it is important to characterize $\PPB$ as made below.

Let $\PPBbar$ be the topological closure of $\PPB$. A $k$-dimensional element of $\PPBbar\setminus \PPB$
is referred to as $k$-face of $\PPB$. There are $2^n$ 0-faces, called corners or vertices.
This set of corners is denoted $\mathcal{C}_{\PPB}$.
Moreover, the subset of $\mathcal{C}_{\PPB}$ obtained with $z_i=1$ is
$\mathcal{C}^1_{\PPB}$ and $\mathcal{C}^0_{\PPB}$ for $z_i=0$. 
The remaining faces of $\PPB$ are parallelotopes. For instance, a $n-1$-dimensional facet of $\PPB$,
say $\F_{i}$, is itself a parallelotope of dimension $n-1$ defined by $n-1$ vectors of $\B$. 
Throughout the paper, the term facet refers to a $n-1$-face. 

A convex polytope (or convex polyhedron) is defined as the intersection of a finite number of half-spaces bounded by hyperplanes \cite{Coxeter1973}:
\[
P_{o}=\{x \in \R^{n} : \ xA \leq b, \ A \in \R^{n \times m}, \ b \in \mathbb{R}^{m}\}.
\]
In this paper, we use not only parallelotopes but also simplices.
A $n$-simplex associated to $\mathcal{B}$ is given by
\begin{align*}
S(\mathcal{B})= \{ & y \in \R^{n} :  \ y= \sum_{i=1}^{n} \alpha_{i}b_{i},
\sum_{i=1}^{n} \alpha_{i} \leq1, \ \alpha_{i} \geq 0 \  \forall \ i   \}.
\end{align*}
It is clear that the corners of $S(\mathcal{B})$, the set $\mathcal{C}_{\SB}$,
are the $n+1$ points $\{0,b_{1}, ... ,b_{n}\}$. 


We say that a function $g : \mathbb{R}^{n-1} \rightarrow \mathbb{R}$ is continuous piecewise linear (CPWL) 
if there exists a finite set of polytopes covering $\mathbb{R}^{n-1}$ (which implies continuity),
and $g$ is affine over each polytope. 
The number of pieces of $g$ is the number of distinct polytopes partitioning its domain.

Finally, $\vee$ and $\wedge$ denote respectively the maximum and the minimum operator. 
We define a convex (resp. concave) CPWL function formed by a set of affine functions
related by the operator $\vee$ (resp. $\wedge$). 
If $\{g_{k}\}$ is a set of $K$ affine functions,
the function $f=g_{1} \vee ... \vee g_{K}$ is CPWL and convex.

\section{The decision boundary function}
\label{sec_dec_bound_func}
Given a VR basis, after translating the point to be decoded inside $\PPB$,
the HLD decoder proceeds in estimating each $z_i$-component separately.
The HLD computes the position of $y$ relative to a boundary via a Boolean equation
to guess whether $z_i=0$, i.e. the closest lattice point belongs to $\mathcal{C}^0_{\PPB}$,
or $z_i=1$ when the closest lattice point is in $\mathcal{C}^1_{\PPB}$.
This boundary cuts $\PPB$ into two regions.
It is composed of Voronoi facets of the corner points. 
The next step is to study the decision boundary function.
\textbf{Without loss of generality, the integer coordinate to be decoded is $z_1$}. 

We recall that a variable $u_{j}(y)$ in the Boolean equations of the HLD is obtained as:
\begin{equation}
\label{eq_bool}
u_{j}(y)=\text{sign}(y \cdot v_j - \bias_j) \in \{ 0, 1\},
\end{equation}
where $v_j$ is the orthogonal vector to the boundary hyperplane $\{ y \in \R^n : \ y \cdot v_j - \bias_j =0  \}$.
The latter contains the Voronoi facet of a point $x \in \CBONE$ and a point from $\CN(x) \cap \CBZERO$.
The decision boundary cutting $\PPB$ into two regions, with $\mathcal{C}^0_{\PPB}$ on one side
and $\mathcal{C}^1_{\PPB}$ on the other side, is the union of these Voronoi facets.
Each facet can be defined by an affine function over a compact subset of $\mathbb{R}^{n-1}$,
and the decision boundary is locally described by one of these functions. 

Let $\{ e_i \}_{i=1}^n$ be the canonical orthonormal basis of the vector space $\R^n$.
For $y \in \R^n$, the $i$-th coordinate is $y_i=y \cdot e_i$.
Denote $\tilde{y}=(y_2, \ldots, y_n) \in \R^{n-1}$ and let $\mathcal{H} = \{ h_j \}$ be the set of affine functions involved in the decision boundary.
The affine boundary function $h_{j}:\R^{n-1} \rightarrow \mathbb{R}$ is
\begin{equation}
\label{equ_hj}
h_{j}(\tilde{y})=y_{1}= \bigg( \bias_{j} -\sum_{k \neq 1} y_{k}v_{j}^{k} \bigg)/v_{j}^{1},
\end{equation}
where $v_{j}^{k}$ is the $k$-th component of vector $v_{j}$.
For the sake of simplicity, in the sequel $h_{j}$ shall denote the function defined
in (\ref{equ_hj}) or its associated hyperplane $\{ y \in \R^n : \ y \cdot v_j - \bias_j =0  \}$
depending on the context.
\begin{theorem}
\label{th_func_VR}
Consider a lattice defined by a VR basis $\B=\{b_i\}_{i=1}^n$.
Suppose that the $n-1$ points $\mathcal{B} \backslash \{ b_1\}$
belong to the hyperplane $\{y \in \R^n : \ y \cdot e_1 =0\}$.
Then, the decision boundary is given by a CPWL function
$f:\R^{n-1} \rightarrow~\R$, expressed as
\begin{equation}
\label{eq_boundary_funct}
f(\tilde{y}) = \wedge_{m=1}^{M}\{\vee_{k=1}^{l_m}g_{m,k}(\tilde{y}) \},
\end{equation}
where $g_{m,k} \in \mathcal{H}$, $1 \leq l_m< \tau_f  $, and  $1 \leq M \leq 2^{n-1}$.
\end{theorem}
In the next theorem, the orientation of the axes relative to $\B$
does not require $\{b_i\}_{i=2}^n$ to be orthogonal to $e_1$,
neither $b_1$ and $e_1$ to be collinear. 
\begin{theorem}
\label{th_func_VR_2}
Consider a lattice defined by a VR basis $\B=\{b_i\}_{i=1}^n$.
Without loss of generality, assume that $b_1^1>0$.
Suppose also that $x_1 > \lambda_1$, $\forall x \in \CBONE$ and $\forall \lambda \in \CN(x) \cap \CBZERO$.
Then, the decision boundary is given by a CPWL function as in (\ref{eq_boundary_funct}).
\end{theorem}
See Appendix~\ref{App_theo_f_CPWL} for the proofs.
Some interesting lattices may not admit a VR basis, e.g. see $E_6$ in~\cite{Corlay2018}.
In this case, if $\text{Vol}(\PPB \setminus \cup_{x \in \CB} \mathcal{V}(x)) \ll \text{Vol}(\PPB)$
then HLD yields efficient decoding.
A basis satisfying this condition is called {\em quasi-Voronoi-reduced}.
The new definition below presumes that $\B$ is quasi-Voronoi-reduced in order
to make a successful discrimination of $z_1$ via the boundary function.
Also, a surface in $\R^n$ defined by a function $g$ of $n-1$ arguments is written as
$\text{Surf}(g)=\{ (g(\tilde{y}), \tilde{y}) \in \R^n : \tilde{y} \in \R^{n-1} \}$. 
\begin{definition}
\label{def_semi-Voronoi-reduced}
Let $\B$ be a basis of $\Lambda$. Assume that  $\B$ and $\{ e_i \}_{i=1}^n$ have the same
orientation as in Theorem~\ref{th_func_VR}. The basis is called
semi-Voronoi-reduced (SVR) if there exists at least two points $x_1, x_2 \in \CBONE$
such that 
$\text{Surf}(\vee_{k=1}^{\ell_1}g_{1,k}) \bigcap \text{Surf}(\vee_{k=1}^{\ell_2}g_{2,k})
\ne \varnothing$, where $\ell_1,\ell_2\geq 1$,
$g_{1,k}$ are the facets between $x_1$ and all points in $\CN(x_1) \cap \CBZERO$,
and $g_{2,k}$ are the facets between $x_2$ and all points in $\CN(x_2) \cap \CBZERO$.
\end{definition}
The above definition of a SVR basis imposes that the boundaries around two points of $\CBONE$,
defined by the two convex functions $\vee_{k=1}^{\ell_m}g_{m,k}$, $m=1,2$,
have a non-empty intersection.
Consequently, the min operator $\wedge$ leads to a boundary function as in (\ref{eq_boundary_funct}).
\begin{corollary}
\label{coro_SVR}
$\PPB$ for a SVR basis $\B$ admits a decision boundary defined by a CPWL function as in (\ref{eq_boundary_funct}).
\end{corollary}

\begin{figure}
    \centering
    \includegraphics[scale=0.4]{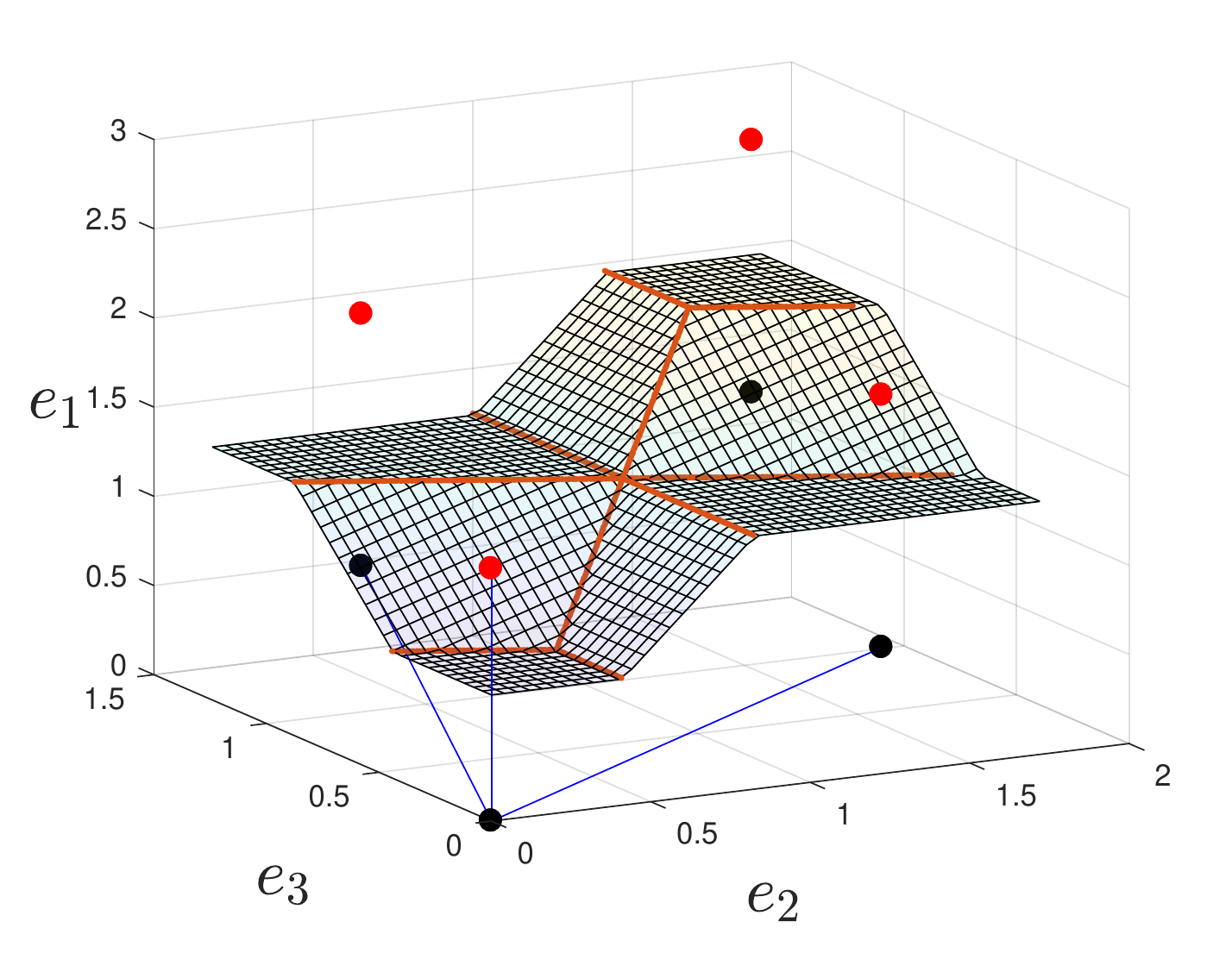}
     \caption{CPWL decision boundary function for $A_3$. 
                   The basis vectors are represented by the blue lines. 
                   The corner points in $\CBONE$ are in red and the corner points in $\CBZERO$ in black.}
     \label{fig_func_A3}
\vspace{-4mm}
\end{figure}
\begin{figure}
    \centering
    \includegraphics[scale=0.4]{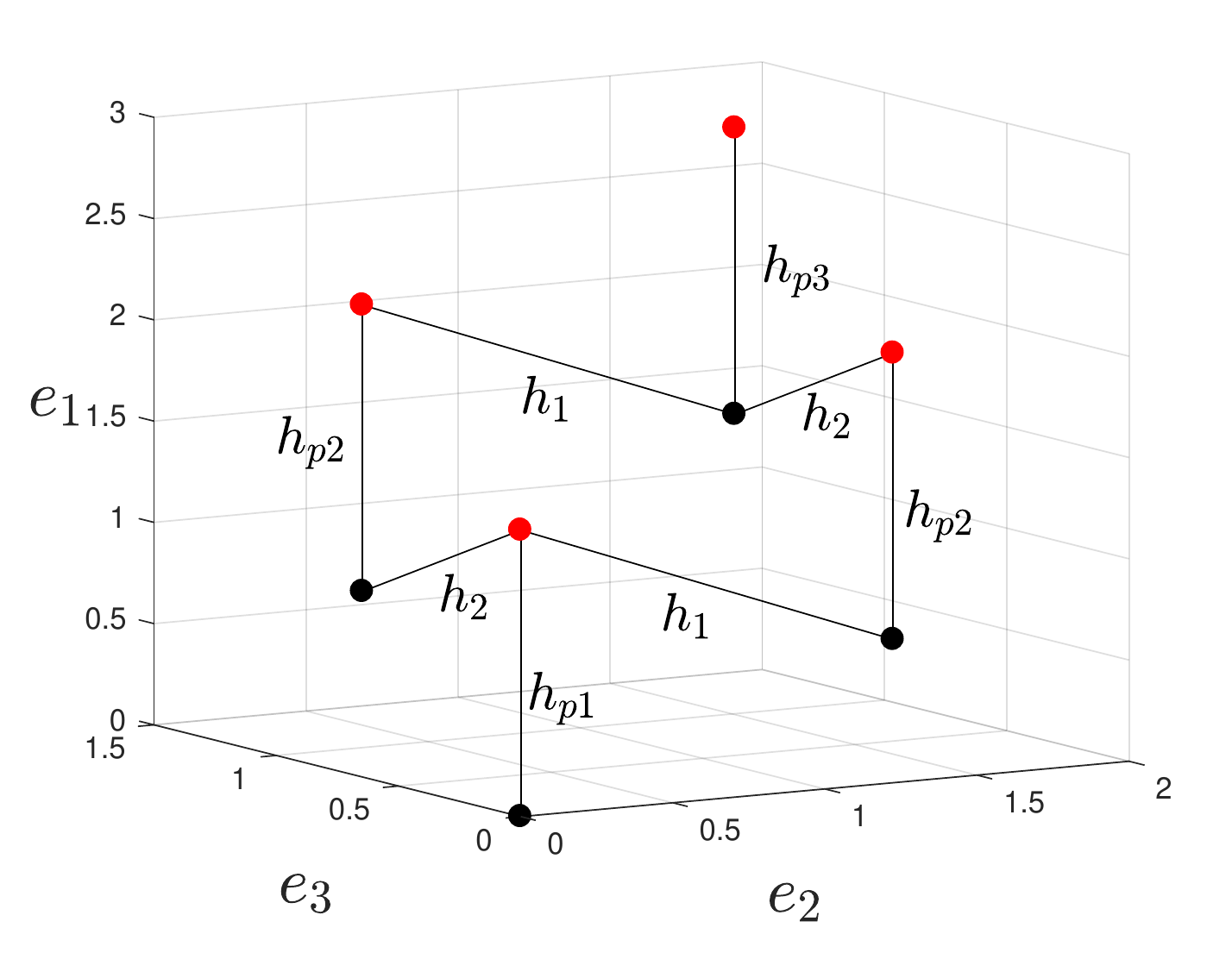}
     \caption{``Neighbor" figure of $\mathcal{C}_{\PPB}$ for $A_3$. Each edge connects a point $x \in \CBONE$ to an element of $\CN(x) \cap \CBZERO$.
		        The $i$ edges connected to a point $x \in \CBONE$ are $1$-faces of a regular $i$-simplex.}
     \label{fig_simplex_A3}
\vspace{-3mm}
\end{figure}
\begin{example}
\label{ex_A3}
Consider the lattice $A_3$ defined by the Gram matrix~\eqref{eq_basis_An}. 
To better illustrate the symmetries we rotate the basis to have $b_1$ collinear with $e_1$. 
Theorem~\ref{th_func_VR_2} ensures that the decision boundary is a function.
The function is illustrated on Figure~\ref{fig_func_A3} and its equation is (we omit the $\tilde{y}$ in the formula to lighten the notations):
\begin{align*}
\label{eq_A3}
f = &\Big[ h_{p_1} \vee h_{1} \vee h_{2} \Big] \wedge \Big[ \left( h_{p_2} \vee h_{1}\right)  \wedge \left( h_{p_2}\vee h_{2} \right) \Big] \wedge  \Big[ h_{p_3} \Big], 
\end{align*}
where $h_{p_1}$, $h_{p_2}$ and $h_{p_3}$ are hyperplanes orthogonal to $b_1$ (the $p$ index stands for plateau). 
On Figure~\ref{fig_simplex_A3} each edge is orthogonal to a local affine function of $f$ and labeled accordingly.
The $[ \cdot ]$ groups all the set of convex pieces of $f$ that includes the same $h_{pj}$. 
Functions for higher dimensions are available in Appendix~\ref{terms_An}.
\end{example}

From now on, {\bf the default orientation of the basis with respect to the
canonical axes of $\R^n$ is assumed to be the one of Theorem~\ref{th_func_VR}}.
We call $f$ the decision boundary function.
The domain of $f$ (its input space) is $\D \subset \mathbb{R}^{n-1}$.
The domain $\mathcal{D}$ is the projection of $\PPB$ on the hyperplane $\{e_i\}_{i=2}^n$.
It is a bounded polyhedron that can be partitioned into convex (and thus connected) regions which we call linear regions. 
For any $\tilde{y}$ in one of these regions,
$f$ is described by a unique local affine function $h_{j}$. 
The number of those regions is equal to the number of affine pieces of $f$. 
\section{Folding-based neural decoding of $A_n$}
In this section, we first prove that the lattice basis
from the $n \times n$ Gram matrix in (\ref{eq_basis_An}) is VR 
(all bases are equivalent modulo rotations and reflections).
We count the number of pieces of the decision boundary function.
We then build a deep ReLU network which computes efficiently this function via $folding$.
Finally, we use the fact that the $n-2$-dimensional hyperplanes
partitioning $\D$ are {\em not} in ``general position" in order to prove
that a two-layer network needs an exponential number of neurons to compute the function.

Consider a basis for the lattice $A_n$ with all vectors from the first lattice shell.
Also, the angle between any two basis vectors is $\pi/3$. Let $J_n$ denote the $n \times n$ all-ones matrix and $I_n$ the identity matrix.
The Gram matrix is
\begin{equation}
\label{eq_basis_An}
\Gamma=GG^T= J_n+I_n.
\end{equation}

\begin{theorem}
\label{theo_An}
A lattice basis defined by the Gram matrix~\eqref{eq_basis_An} is Voronoi-reduced.
\end{theorem}

See Appendix~\ref{App_theo_VR_An} for the proof. 
Consequently, discriminating a point in $\PPB$
with respect to the decision boundary leads to an optimal decoder. 

\subsection{Number of pieces of the decision boundary function}

We count the number of pieces, and thus linear regions,
of the decision boundary function $f$.
We start with the following lemma involving $i$-simplices.


\begin{lemma}
\label{lem_simplex}
Consider an $A_n$-lattice basis defined by the Gram matrix~\eqref{eq_basis_An}.
The decision boundary function $f$ has a number of affine pieces equal to
\begin{align}
\sum_{i=0}^{n}i \times ( \text{$\#$ regular $i$-simplices}),
\end{align}
where, for each $i$-simplex, only one corner $x$ belongs to $\CBONE$
and the other corners constitute the set $\CN(x) \cap \CBZERO$.
\end{lemma}
\begin{proof}
A key property of this basis is 
\begin{gather}
\label{eq_prop}
\begin{split}
&\forall x \in~\CBZERO, \ x' \in A_n \backslash \{b_j,0\},\ 2 \leq j\leq n: \\
&x+ b_j \in \CN(x+b_1), \ x+ x' \not\in \CN(x+b_1) \cap \CBZERO.
\end{split}
\end{gather}
It is obvious that $\forall x \in~\CBZERO$: $x + b_1 \in~\CBONE$. 
This implies that any given point $x \in \mathcal{C}^1_{\PPB}$
and its neighbors $\CN(x)~\cap~\CBZERO$ form a regular simplex
$\mathcal{S}$ of dimension $|\CN(x) \cap \CBZERO|$. 
This clearly appears on Figure~\ref{fig_simplex_A3}.
Now, consider the decision boundary function of a $i$-simplex separating the top corner (i.e. $\mathcal{C}^1_{\mathcal{S}}$) from all the other corners  (i.e. $\mathcal{C}^0_{\mathcal{S}}$). 
This function is convex and has $i$ pieces. 
The maximal dimensionality of such simplex is obtained
by taking the points 0, $b_1$, and the $n-1$ points $b_j$, $j\ge 2$.
\end{proof}

\begin{theorem}
\label{theo_nbReg_Lin}
Consider an $A_n$-lattice basis defined by the Gram matrix~\eqref{eq_basis_An}.
The decision boundary function $f$ has a number of affine pieces equal to
\begin{equation}
\label{eq_An_pieces}
\sum_{i=1}^{n}i \times \binom{n-1}{n-i}.
\end{equation}
\end{theorem}

\begin{proof}
From Lemma~\ref{lem_simplex},
what remains to be done is to count the number of $i$-simplices.
We walk in $\CBZERO$ and for each of the $2^{n-1}$ points $x \in \CBZERO$ 
we investigate the dimensionality of the simplex where the top corner is $x + b_1 \in \CBONE$. 
This is achieved by counting the number of elements in $\CN(x+~b_1)~\cap~\CBZERO$, 
via the property given by (\ref{eq_prop}). 
Starting from the origin, one can form a $n$-simplex with $0$, $b_1$, and the $n-1$ other basis vectors. 
Then, from any $b_{j_1}$, $2 \leq j_1\leq n$, one can only add the $n-1$ remaining basis vectors to generate a simplex in $\PPB$. 
Indeed, if we add again~$b_{j_1}$, the resulting point is outside of $\PPB$. 
Hence, we get a $n-1$-simplex and there are $\binom{n-1}{1}$
ways to choose $b_{j_1}$: any basis vectors except~$b_1$. 
Similarly, if one starts the simplex from $b_{j_1}+b_{j_2}$, $\j_1 \neq j_2$, one can form a $n-2$-simplex in $\PPB$ and there are $\binom{n-1}{2}$ ways to choose $b_{j_1}+b_{j_2}$. 
In general, there are $\binom{n-1}{k}$ ways to form a $n-k$-simplex.
Applying the previous lemma and summing over $k=n-i=0 \ldots n-1$ gives
the announced result.
\end{proof}


\subsection{Decoding via folding}
\label{sec_folding_An}
Obviously, at a location $\tilde{y}$,
we do not want to compute all affine pieces in (\ref{eq_boundary_funct})
whose number is given by (\ref{eq_An_pieces}) in order to evaluate $f$.
To reduce the complexity of this evaluation,
the idea is to exploit the symmetries of $f$
by ``folding" the function and mapping distinct regions
of the input domain to the same location.
If folding is applied sequentially, i.e. fold a region that has already been folded,
it is easily seen that the gain becomes exponential. 
The notion of folding the input space in the context of neural networks
was introduced in \cite{Montufar2014}.

Given the basis orientation as in~Theorem~\ref{th_func_VR},
the projection of $b_j$ on $\mathcal{D}$ is $b_j$ itself, for $j\ge 2$. 
We also denote the bisector hyperplane
between two vectors $b_j, b_k$ by $BH(b_j, b_k)$
and its normal vector is taken to be $v_{j,k}=~b_j-b_k$. 
We define the folding transformation $F:\D \rightarrow \D'$ as follows:  
let $\tilde{y} \in \D$, for all $2\le j < k \le~n$,
compute $\tilde{y} \cdot v_{j,k}$ (the first coordinate of $v_{j,k}$ is zero).
If the scalar product is non-positive, replace $\tilde{y}$
by its mirror image with respect to $BH(b_j, b_k)$.
There exist $(n-1)(n-2)/2$ hyperplanes for mirroring. 

\begin{theorem}
\label{theo_An_linear}
Let us consider the lattice $A_n$ defined by the Gram matrix~\eqref{eq_basis_An}. 
We have (i) $\D' \subset \D$, (ii) for all $\tilde{y} \in \D$, $f(\tilde{y}) = f(F(\tilde{y}))$ and (iii) $f$ has exactly 
\vspace{-1.5mm}
\begin{equation}
2n+1
\vspace{-1.5mm}
\end {equation}
pieces on $\D'$.
This is to be compared with \eqref{eq_An_pieces}.
\end{theorem}
\vspace{-1mm}
See Appendix~\ref{App_theo_An_linear} for the proof.
\vspace{-1mm}
\begin{example}[Continued]
\label{ex_A3_2}
The function $f$ restricted to $\D'$ (i.e. the function to evaluate after folding), say $f_{\D'} $, is
\begin{align}
\label{eq_A3}
f_{\D'} = \Big[ h_{p1} \vee h_{1} \Big] \wedge \Big[  h_{p2} \vee h_{2}  \Big] \wedge \Big[ h_{p3} \Big]. 
\end{align}
The general expression of $f_{\D'} $ for any dimension is available in Appendix~\ref{terms_An}.
\end{example}

\subsection{From folding to a deep ReLU network}
\label{sec_folding_relu}
For sake of simplicity and without loss of generality, in addition to the standard ReLU activation function ReLU$(a)=\max(0, a)$, we also allow the function $\max(0,-a)$ and the identity as activation functions in the network.

To implement a reflection, one can use the following strategy. 
Step~1: rotate the axes to have the $i$-th axis $e_i$ perpendicular to the reflection hyperplane and shift the point (i.e. the $i$-th coordinate) to have the reflection hyperplane at the origin.
Step~2: take the absolute value of the $i$-th coordinate. 
Step~3: do the inverse operation of step 1.

Now consider the ReLU network illustrated in Figure~\ref{fig_relu_ref}. 
The edges between the input layer and the hidden layer represent the rotation matrix, where the $i$-th column is repeated twice, and $p$ is a bias applied on the $i$-th coordinate.
Within the dashed square, the absolute value of the $i$-th coordinate is computed
and shifted by $-p$.
Finally, the edges between the hidden layer and the output layer represent the inverse rotation matrix. 
This ReLU network computes a reflection. We call it a reflection block.

All reflections can be naively implemented by a simple
concatenation
of reflection blocks. 
This leads to a very deep 
\newpage
\noindent
and narrow network
of depth $\mathcal{O}(n^2)$ and width $\mathcal{O}(n)$.
\begin{figure}
    \centering
    \includegraphics[scale=0.75]{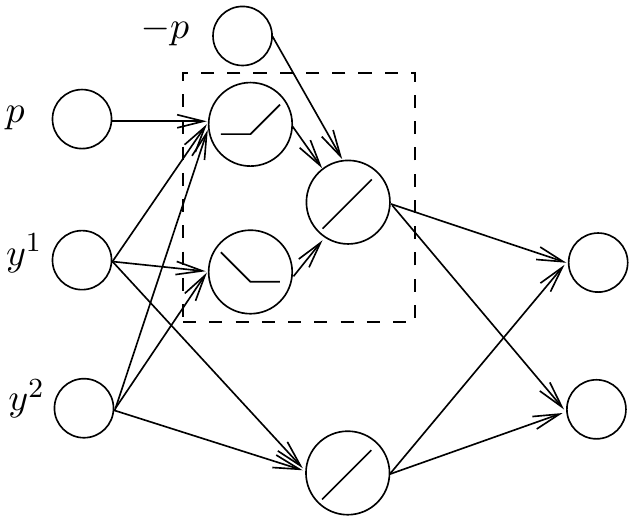}
     \caption{Reflection ReLU network (called reflection block).}
     \label{fig_relu_ref}
	\vspace{-4mm}
\end{figure}

Regarding the $2n+1$ remaining pieces after folding, we have two options 
(in both cases, the number of operations involved is negligible compared to the previous folding operations). 
To directly discriminate the point with respect to $f$, 
we implement the HLD on these remaining pieces with two additional hidden layers 
(see e.g. Figure~2 in \cite{Corlay2018}): 
project $y_{folded}$ on the $2n+1$ hyperplanes (with one layer of width $2n +1$) 
and compute the associated Boolean equation with an additional hidden layer.
If needed, we can alternatively evaluate $f(\tilde{y})$ via $\mathcal{O}(\log(n))$ additional hidden layers. 
First, compute the $n-1$ 2-$\vee$ via two layers of size $\mathcal{O}(n)$ containing several ``max ReLU networks" 
(see e.g. Figure 3 in \cite{Arora2018}). Then, compute the $n$-$\wedge$ via $\mathcal{O}(\log(n))$ layers. 
Note that $f(\tilde{y})$ can also be used for discrimination via the sign of $y_{i}-f(\tilde{y})$. 

The total number of parameters in the whole network is $\mathcal{O}(n^4)$.
In Appendix~\ref{App_number_params}, 
we quickly discuss whether or not this can be improved.

Eventually, the CPV is solved by using $n$ such networks in parallel (this could also be optimized). 
The final network has width $\mathcal{O}(n^2)$ and depth $\mathcal{O}(n^2)$.
\subsection{Decoding via a shallow network}
\label{sec_shallow}
A two-layer ReLU network with $n$ inputs and $w_1$ neurons in the hidden layer can compute a CPWL function with at most $\sum_{i=0}^{n}\binom{w_1}{i}$ pieces \cite{Pascanu2013}. 
This result is easily understood by noticing that the non-differentiable part of $\max(0,a)$ is a $n-2$-dimensional hyperplane that separates two linear regions. 
If one sums $w_1$ functions $\max(0,d_{i} \cdot y)$, where $d_i$, $1\leq i \leq w_1$, is a random vector, one gets $w_1$ of such $n-2$-hyperplanes.  
The rest of the proof consists in counting the number of linear regions that can be generated by these $w_1$ hyperplanes. 
The number provided by the previous formula is attained if and only if the hyperplanes are in general position. 
Clearly, in our situation the $n-2$-hyperplanes partitioning $\D$ are not in general position: the hyperplane arrangement is not $simple$.
The proof of the following theorem, available in Appendix~\ref{App_theo_shallow}, consists in finding a lower bound on the number of such $n-2$-hyperplanes.
\vspace{-1mm}
\begin{theorem}
\label{theo_shallow}
A ReLU network with one hidden layer needs at least 
\vspace{-2mm}
\begin{equation}
\label{eq_An_shallow}
\sum_{i=2}^{n}(i-1) \times \binom{n-1}{n-i}
\end{equation}
neurons to solve the CVP for the lattice $A_n$.
\end{theorem}

\section{Conclusions}
\vspace{1mm}
We recently applied this theory to a SVR basis of lattices $E_n$, $6~\leq~n~\leq 8$. 
The decision boundary function has a number of pieces equal to

\footnotesize
\begin{gather*}
\sum_{i=0}^{n-3} \Bigg( \left[ 1+(n-3-i) \right] + 2 \left[ 1+ 2(n-3-i)+ \binom{n-3-i}{2} \right] + \\ 
\left[ 1+ 3(n-3-i) +   3\binom{n-3-i}{2} +  \binom{n-3-i}{3} \right] \Bigg)\binom{n-3}{n-i}-3.
\end{gather*}
\normalsize
a number which we successfully linearized via folding.

From a learning perspective, our findings suggest that many optimal decoders may be contained in the function class $\Phi$ of deep ReLU networks. Learnability results of the restricted model \cite{Anthony1999}\cite{Bartlett2017} show that the sample complexity is then $m_{\Phi}(\epsilon,\delta)=\mathcal{O}(WL \log(W)/\epsilon)$, avoiding the $1/\epsilon^2$ of the general model (where $W$ is the number of parameters in the network and $L$ the number of layers).

Additionally, the folding approach suits very well the non-uniform finite sample bounds of the information bottleneck framework \cite{Shamir2010}. Indeed, if we model the input of the network and its $i$-th layer after the $i$-th reflection block by random variables $Y$ and $Y_i$, clearly $I(Y;Y_i)$ is reduced compared to $I(Y;Y_{i-1})$, for any distribution of $Y$.


\clearpage



\section{Appendix}

\subsection{First order terms of the decision boundary function before and after folding for $A_n$}
\label{terms_An}

\subsubsection{Before folding}

\begin{align*}
f^{n=2} = \Big[   h_{p1} \vee h_{1}  \Big] \wedge \Big[ h_{p2} \Big].
\end{align*}
\begin{align*}
f^{n=3}= & \Big[   h_{p1} \vee h_{1} \vee h_{2}  \Big] \ \wedge \\
& \Big[ \left( h_{p2} \vee h_{1}\right)  \wedge \left( h_{p2}\vee h_{2} \right)  \Big] \ \wedge \\
& \Big[   h_{p3} \Big]. 
\end{align*}
\begin{align*}
f^{n=4} = & \Big[ h_{p1} \vee h_{1} \vee h_{2} \vee h_{3} \Big] \wedge \\
            & \Big[ \left( h_{p2}  \vee h_{1} \vee h_{2}\right) \wedge \left( h_{p2} \vee h_{2} \vee h_{3} \right) \Big] \wedge \\ 
            & \left( h_{p2} \vee h_{1} \vee h_{3} \right)\Big] \wedge \\
            & \Big[ \left( h_{p3} \vee h_{1} \right) \wedge\left( h_{p3} \vee h_{2}\right) \wedge \left( h_{p3} \vee h_{3}\right)\Big] \wedge \\
            &  \Big[ h_{p4}\Big].
\end{align*}
\subsubsection{After folding, $\forall n \ge 2$}
\begin{align*}
f^n_{\D'} = & \Big[  h_{p1} \vee h_{1}  \Big] \wedge \Big[ h_{p2} \vee h_{2} \Big] \wedge ... \ \wedge   \\
          & \Big[  h_{p (n-1)} \vee h_{n-1} \Big] \wedge \Big[ h_{p (n)} \Big].
\end{align*}

\subsection{Proof of Theorem~\ref{th_func_VR}~$\&$~\ref{th_func_VR_2}}
\label{App_theo_f_CPWL}

Note that the assumptions of Theorem~\ref{th_func_VR_2} are more general than the ones of Theorem~\ref{th_func_VR}:
the orientation of the axes chosen for Theorem~\ref{th_func_VR} always satisfies: $x_1 > \lambda_1$, $\forall x \in \CBONE$ and $\forall \lambda \in \CN(x) \cap \CBZERO$ (if $b_1^1$ is negative, the equivalent assumption is: $x_1 < \lambda_1$, $\forall x \in \CBONE$ and $\forall \lambda \in \CN(x) \cap \CBZERO$).
Indeed, with this orientation, any point in $\CBZERO$ is in the hyperplane $\{y \in \R^n : \ y \cdot e_1 =0\}$ and has its first coordinate equal to 0. 
As a result, the proof of Theorem~\ref{th_func_VR_2} (below) also proves Theorem~\ref{th_func_VR}.

\begin{proof}
All Voronoi facets of $f$ belonging to a same point of $\CBONE$ form a polytope.
The variables within a AND condition of the HLD discriminate a point with respect to the boundary hyperplanes where these facets lie: 
the condition is true if the point is on the proper side of all these facets.
For a given point $y \in \PPB$, we write a AND condition $m$ as sign($yA_m + q_m)\succ 0$, where $A_m \in \mathbb{R}^{ n \times l_m }$, $ q_m \in \R^{l_m}$.
Does this convex polyhedron lead to a convex CPWL function?

Consider Equation~$\eqref{eq_bool}$. 
The direction of any $v_j$ is chosen so that the Boolean variable is true for the point in $\CBONE$ whose Voronoi facet is in the corresponding boundary hyperplane.
Obviously, there is a boundary hyperplane, which we name $\psi$, between the lattice point $0 \in \CBZERO$ and $b_1 \in \CBONE$. 
This is also true for any $x \in \CBZERO$ and $x + b_1 \in \CBONE$.
Now, assume that one of the vector $v_j$ has its first coordinate $v^1_j$ negative.
It implies that for a given location $\tilde{y}$, if one increases $y_1$ the term $y \cdot v_{j}^{T} - \bias_{j}$ decreases and eventually becomes negative if it was positive.
Note that the Voronoi facet corresponding to this $v_j$  is necessarily above $\psi$, with respect to the first axis $e_1$, as the Voronoi cell is convex.
It means that there exists $\tilde{y}$ where one can do as follows. 
For a given $y_1$ small enough, $y$ is in the decoding region $z_1=0$.
If one increases this value, $y$ will cross $\psi$ and be in the decoding region  $z_1=1$. 
If one keeps increasing the value of $y_1$, $y$ eventually crosses the second hyperplane and is back in the region $z_1=0$.
In this case $f$ has two different values at the location $\tilde{y}$ and it is not a function.
If no $v^1_j$ is negative, this situation is not possible.
All $v^1_j$ are positive if and only if all  $x \in \CBONE$ have their first coordinates $x_1$ larger than the first coordinates of all $\CN(x) \cap \CBZERO$.
Hence, the convex polytope leads to a function if and only if this condition is respected.
If this is the case, we can write sign($yA_m + q)\succ 0 \Leftrightarrow \wedge_{k=1}^{l_m} y \cdot a_{m,k} + q_{m,k} > 0$, $a_{m,k},q_{m,k} \in \{v_j, \bias_j\}$. 
We want $y_{1}>g_{m,k}(\tilde{y})$, for all $1\leq k \leq l_m$, which is achieved if $y_{1}$ is greater than the maximum of all values. 
The maximum value at a location $\tilde{y}$ is the active piece in this convex region and we get $y_{1} = \vee_{k=1}^{l_m} g_{m,k}(\tilde{y})$. 

A Voronoi facet of a neighboring Voronoi cell is concave with the facets of the other Voronoi cell it intersects.
The region of $f$ formed by Voronoi facets belonging to distinct points in $\CBONE$ form concave regions that are linked by a OR condition in the HLD. The condition is true if $y$ is in the Voronoi region of at least one point of $\CBONE$: $\vee_{m=1}^{M} \{ \wedge_{k=1}^{l_m} y \cdot a_{m,k} + q_{m,k} \}> 0$. We get $f(\tilde{y}) = \wedge_{m=1}^{M}\{\vee_{k=1}^{l_m}g_{m,k}(\tilde{y}) \}$.

Finally, $l_m$ is strictly inferior to~$\tau_{f}$ because all Voronoi facets lying in the affine function of a convex part of $f$ are facets of the same corner point. 
Regarding the bound on $M$, the number of logical OR term is upper bounded by half of the number of corner of $\PPB$ which is equal to $2^{n-1}$. 
\end{proof}

\subsection{Proof of Theorem~\ref{theo_An}}
\label{App_theo_VR_An}

\begin{proof}
We need to show that none of $\mathcal{V}(x)$, $x \in \Lambda \backslash \mathcal{C}_{\PPB}$, crosses a facet of $\PPBbar$ (the closure of $\PPB$).
In this scope, we first find the closest points to a facet and show that its Voronoi region do not cross $\PPBbar$.
It is sufficient to proof the result for one facet of $\PPBbar$ has the landscape is the same for all of them.

Let $HF_1$ denote the hyperplane defined by $\mathcal{B} \backslash b_1$, where the facet $\mathcal{F}_{1}$ of $\PPBbar$ lies.
While $b_{1}$ is in $\PPBbar$ it is clear that $-b_1$ is not in $\PPBbar$. 
Adding to $-b_1$ any linear combination of the $n-1$ vectors generating $\mathcal{F}_{1}$ is equivalent to moving in a hyperplane, 
say $HP_1$, parallel to $\mathcal{F}_{1}$ and it does not change the distance from $HF_1$. 
Additionally, it is clear that any integer multiplication of $-b_{1}$ results in a point which is further from the hyperplane (except by $\pm 1$ of course). 
Note however that the orthogonal projection of $-b_1$ onto $HF_1$ is not in $\mathcal{F}_{1}$. The only lattice point in $HP_1$ having this property is obtained by adding all $b_j$, $2 \leq j \leq n $, to $-b_1$, i.e. the point $-b_1 + \sum_{j=2}^n b_j$.

This closest point to $\PPBbar$, along with the points $\mathcal{B} \backslash b_1$, form a regular simplex. Hence, the hole of the Voronoi region of interest is the centroid of this regular simplex (it is not a deep hole of $A_n$ for $n \geq 3$). It is located at a distance of  $\alpha/(n+1)$, $\alpha>0$,  to the center of any facet of the simplex and thus to $\mathcal{F}_1$ as well as to $\PPBbar$.
\end{proof}
\subsection{Proof of Theorem~\ref{theo_An_linear}}
\label{App_theo_An_linear}

\begin{proof}
To prove (i) we  use the fact that $BH(b_j,b_k)$, $2\le  j < k \le~n$,
is orthogonal to  $\D$, then the image of $\tilde{y}$  via the folding
$F$ is  in $\D$.   

(ii) is the  direct result of  the symmetries  in the
$A_n$  basis where  the  $n$ vectors  form  a regular  $n$-dimensional
simplex.  The folding  via $BH(b_j, b_k)$ switches $b_j$  and $b_k$ in
the hyperplane containing $\D$ and  orthogonal to $e_1$. Switching $b_j$
and $b_k$ does  not change the decision boundary because  of the basis
symmetry, hence  $f$ is unchanged.

Now, for (iii), how  many pieces are left after all reflections?
Similarly to the proof of Theorem~\ref{theo_nbReg_Lin}, 
we walk in $\CBZERO$ and for a given point $x \in \CBZERO$ we investigate the dimensionality of the simplex 
where the top corner is $x + b_1 \in \CBONE$.
This is achieved by counting the number of elements in $\CN(x+b_1)~\cap~\CBZERO$ (via Equation~\ref{eq_prop}) 
that are on the proper side of all bisector hyperplanes.
Starting from the origin, one can only form a $2$-simplex with with $0$, $b_1$, and $b_2$: 
any other point $b_{j}$, $j \ge 3$, is on the  other  side  of  the  the bisector  hyperplanes $BH(b_2, b_j)$.
Hence, the lattice  point $b_1$, which had
$n$ neighbors in $\CBZERO$ before folding, only has  2 now. 
$f$ has only two pieces around $b_1$ instead of $n$.
Then, from  $b_{2}$ one can add $b_{3}$ but no other for
the same reason. 
The point $b_2+b_1$ has only 2  neighbors in  $\CBZERO$.
The pattern replicates until the last corner reaching
$b_1+b_2+\ldots+b_n$ which has only one neighbor.
So we get $2(n-1)+1$ pieces.
\end{proof}
\subsection{Proof of Theorem~\ref{theo_shallow}}
\label{App_theo_shallow}

The proof below complements the discussion of Section~\ref{sec_shallow}: we provide a lower bound on the number of { \em distinct } $n-2$-hyperplanes 
(or more accurately the $n-2$-faces located in $n-2$-hyperplanes) partitioning $\D$. 
Note that these $n-2$-faces are the projections in $\D$ of the 
$n-2$-dimensional intersections of the affine pieces of $f$.
\begin{proof}
We show that many intersections between two affine pieces linked by a $\vee$ operator (i.e. an intersection of affine pieces within a convex region of $f$) are located in distinct $n-2$-hyperplanes.
To prove it, consider all $2$-simplices in $\PPB$ of the form $\{x_1, x_1 + b_1, x_1+ b_j\}$, $x_1 \in \CBZERO$, $x_1+ b_j \in \CBZERO$.
The decision boundary function of any of these simplices has 2 pieces and their intersection is a $n-2$-hyperplane.   
Take a simplex among these simplices, say $\{0,b_1, b_2\}$.  Any other simplex is obtained by a composition of reflections and translations from this simplex. 
For two $n-2$-hyperplanes to be the same, a second simplex should be obtained from the first one by a translation along a vector orthogonal to the $2$-face of this first simplex (i.e. a vector parallel to the $n-2$-hyperplane). 
However, the allowed translations are only in the the direction of a basis vector. None of them is orthogonal to one of these simplices.  
 
Finally, note that any $i$-simplex encountered in the proof of 
Theorem~\ref{theo_nbReg_Lin} can be decomposed into $i-1$ of such 2-simplex. 
 Hence, from the proof of Theorem~\ref{theo_nbReg_Lin}, 
we get that the number of this category of 2-simplex, and thus a lower bound on the number of $n-2$-hyperplanes, 
is $\sum_{k=0}^{n-1}(n-1-k)~\binom{n-1}{k}$. Summing over $k=n-i=0 \ldots n-1$ gives
the announced result.
\end{proof}

\subsection{The function of \cite{Montufar2014} from a lattice perspective}
\label{sec_montu}


The function constructed in \cite{Montufar2014} can be simply described as follows. 
Consider a typical random CPWL function $f: \R^{n-1} \rightarrow \R$ in the cube $[0,1]^{n}$, having a number of pieces meeting the bound mentioned in Section~\ref{sec_shallow}.
Take the mirror image of this cube with respect to all $2^{n-1}-1$ possible combinations of hyperplanes $\{y \in \R^n : \ y \cdot e_i =1\}$, $2 \leq i \leq n$. Finally, tile a subset of $\R^n$ in the direction of all axes, except $e_1$, with this fundamental block of $2^{n-1}$ cubes. 

This is the same idea as Construction A \cite{Conway1999} except that the function within the cube $[0,1]^{n}$ is different than the one obtained with a code (e.g. see Figure~\ref{fig_parity}). 
In \cite[chap. 20]{Conway1999}, they present a method to find the equivalent position of any point in $\R^n$ in the cube $[0,1]^{n}$: 
first, perform a mod 2 operation on each scalar. 
If the resulting point is still outside of the cube, perform the needed reflections with respect to the sides of the cube (i.e. if $1<y_i<2$, replace $y_i$ by $2-y_i$). 

\begin{figure}
     \includegraphics[width=0.8\linewidth]{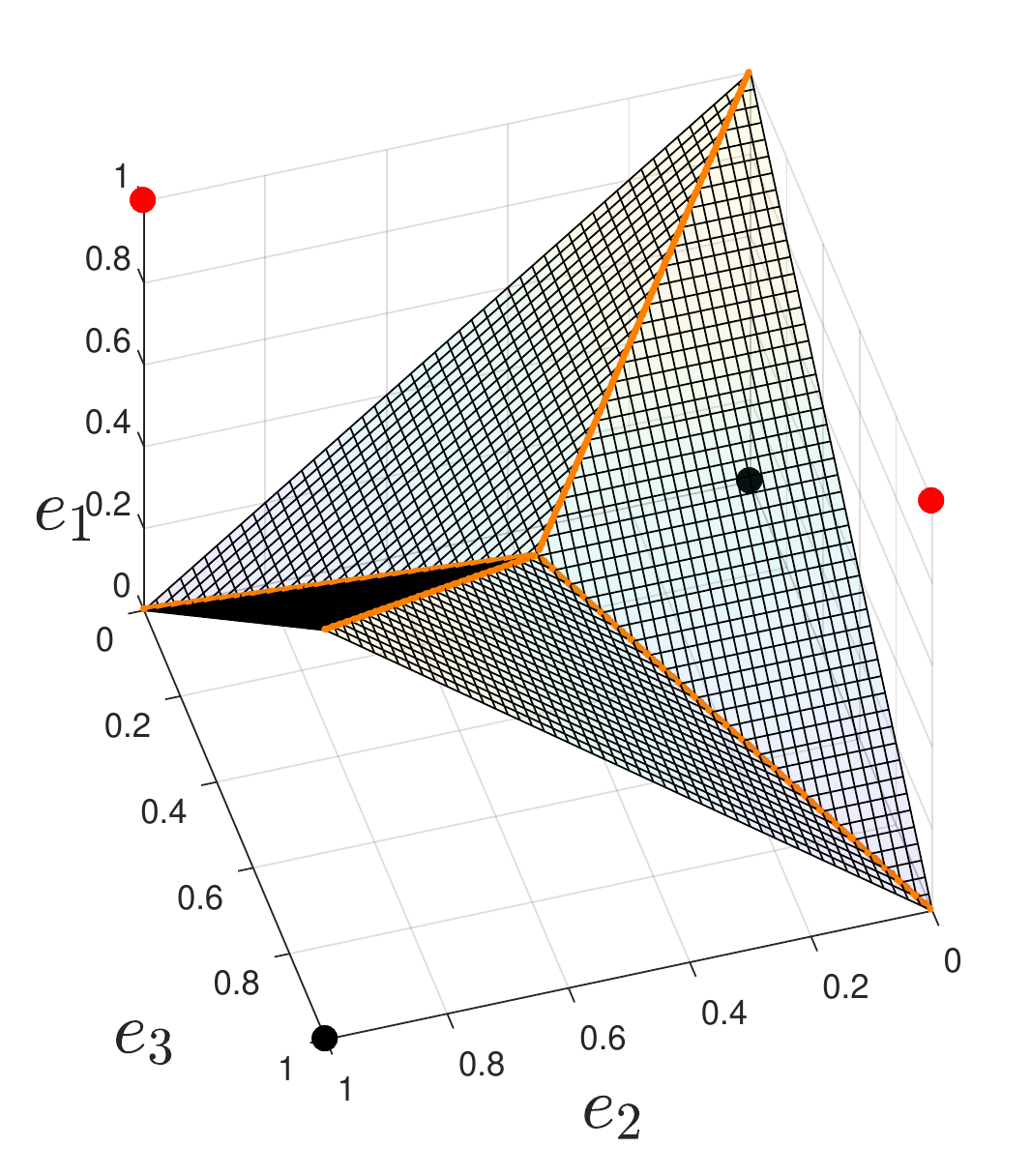}
     \caption{Decision boundary function for the parity check code used to obtain $D_3$ via Construction A.}
     \label{fig_parity}
 \end{figure}

Hence, to compute the function of \cite{Montufar2014} (neglecting the last layer of their neural network) 
one should first perform a mod 2 operation on each axis and then perform a reflection 
along the necessary axes to end up with a point in the cube $[0,1]^{n}$. 
This essentially amounts to computing the sawtooth function (see Figure \ref{fig_sawtooth}) on each axis 
in order to implement the mod 2 operation.

This function can be well approximated via the highest harmonics of its Fourier series, i.e. by summing sines, or simply by rotating a non-symmetric triangle wave function. 
Therefore, the main argument of their proof 
(this function is used to prove a lower bound on the maximum number of linear regions that can be achieved by a deep ReLU network) 
lies on the fact that instead of summing $p$ sigmoids to create $p/2$ periods of a sine it is more efficient to fold it $\log_{2}(p)$ times similarly to what can be done with a sheet of paper (note that \cite{Telgarsky2016} also uses a similar sawtooth function to compute its bounds).
Even though it is inspiring, it hardly justifies the superiority of deep learning over conventional methods as each axis can be processed independently. 

\begin{figure}
	\centering
     \includegraphics[width=0.7\linewidth]{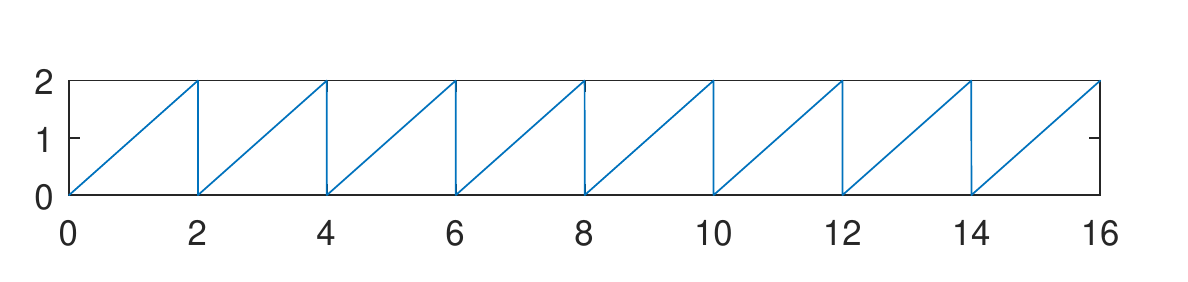}
    \vspace{-2mm}
     \caption{Sawtooth function.}
    \vspace{-3mm}
     \label{fig_sawtooth}
 \end{figure}

In a sense, our work is complementary. 
Indeed, if one is allowed additional layers, we show that even when the ``fundamental" cube is reached, for some functions, one can keep folding.

\subsection{Number of parameters in the deep ReLU network}
\label{App_number_params}

In Section~\ref{sec_folding_relu}, the reflexions are naively implemented 
by a simple concatenation of $\mathcal{O}(n^2)$ reflection blocks.
Can we do better?
The $n-1$ coordinates of $v_{j,k}$ imply $\mathcal{O}(n^2)$ parameters
(i.e. edges) per reflection block.  
But many $v_{j,k}$ can be orthogonal to several axes $e_i$ for some orientations of the basis.

Consider a lower triangular basis with $b_1$ collinear to $e_1$.
Among the sum of all coordinates of all $v_{j,k}$ only $\sum_{j=1}^{n-2}\sum_{k=2}^{j+1} k$ 
coordinates are non-null. 
This means that the depth can be reduced and/or the network is sparse. 
Unfortunately, the previous expression is still $\mathcal{O}(n^3)$
and the number of parameters per reflection block $\mathcal{O}(n^2)$.
The total number of parameters in the network remains $\mathcal{O}(n^4)$.

\end{document}